\documentclass{article}

\usepackage[preprint]{neurips_2026}

\usepackage[utf8]{inputenc} 
\usepackage[T1]{fontenc}    
\usepackage[colorlinks=true,allcolors=blue]{hyperref}       
\usepackage{url}            
\usepackage{booktabs}       
\usepackage{amsfonts}       
\usepackage{nicefrac}       
\usepackage{microtype}      
\usepackage{xcolor}         
\usepackage{wrapfig}

\usepackage{amsmath}
\usepackage{amsthm}
\usepackage{mathtools}
\usepackage[mathscr]{euscript}
\usepackage{dsfont}

\let\P\undefined
\DeclareMathOperator*{\P}{\mathbb{P}}
\DeclareMathOperator*{\E}{\mathbb E}
\DeclareMathOperator{\tr}{tr}
\DeclareMathOperator{\rank}{rank}
\DeclareMathOperator{\range}{range}

\DeclarePairedDelimiter{\bracket}{[}{]}
\DeclarePairedDelimiter{\paren}{(}{)}

\newcommand{\sX}{{\mathscr X}}
\newcommand{\sY}{{\mathscr Y}}
\def\Rset{\mathbb{R}}
\newcommand{\1}{\mathds{1}}

\newcommand{\ignore}[1]{}

\theoremstyle{plain}
\newtheorem{theorem}{Theorem}[section]

\newtheorem{proposition}[theorem]{Proposition}
\newtheorem{fact}[theorem]{Fact}

\usepackage{subcaption}
\usepackage{xcolor}

\title{A Bitter Lesson for Data Filtering}

\author{%
  Christopher Mohri \\
  Department of Computer Science\\
  Stanford University\\
  \texttt{xmohri@stanford.edu} \\
  \And
  John Duchi \\
  Departments of Statistics and Electrical Engineering \\
  Stanford University\\
  \texttt{jduchi@stanford.edu} \\
  \And
  Tatsunori Hashimoto \\
  Department of Computer Science\\
  Stanford University\\
  \texttt{thashim@stanford.edu} \\
}

\begin{document}

\maketitle

\begin{abstract}
    We investigate data filtering for large model pretraining via new scaling studies that target the high compute, data-scarce regime. In spite of an apparently common belief that filtering data to include only high-quality information is essential, our experiments suggest that with enough compute, the best data filter is no data filter. We find that sufficiently trained large parameter models not only tolerate low-quality and distractor data, but in fact benefit from nominally ``poor'' data.

\end{abstract}

\section{Introduction}

The standard approach to select pretraining data for language models is to filter text from sources like Common Crawl (CC) \citep{commoncrawl}. It is widely documented that in compute-constrained regimes, where one must train on a subset of CC, different data selection strategies can have a large impact on performance. This is intuitive: all else equal, it seems natural to train on ``higher-quality'' data. As a result, a large body of research has emerged to tackle the data selection problem, with the goal of finding the best subset for pretraining language models \citep{albalak2024surveydataselectionlanguage,li2025datacomplmsearchgenerationtraining}.  

However, not only is large-scale filter ablation heuristic and expensive, but filtering removes data, which is at odds with scaling trends that prescribe ever-increasing amounts of data to improve model performance. For example, the heavily-filtered DCLM-Baseline dataset keeps $\sim\!1\%$ of the original CC, leading to about 3.8 trillion tokens \citep{li2025datacomplmsearchgenerationtraining}. While this is still enormous, it falls short of the Chinchilla-optimal token budget for a 1 trillion parameter model, even after accounting for diminishing returns when epoching \citep{muennighoff2025scalingdataconstrainedlanguagemodels}. The current trend is also to over-train relative to Chinchilla-optimal, which prescribes even more tokens to allow for (relatively) smaller models that are financially feasible to serve \citep{sardana2025chinchillaoptimalaccountinginferencelanguage}. 

We begin by testing the hypothesis that data filtering is necessary at all in the large compute limit. While large-scale machine learning has moved toward task-agnostic pretraining  \citep{raffel2023exploringlimitstransferlearning}, and there is anecdotal evidence that larger computational budgets benefit from looser data filters \citep{goyal2024scalinglawsdatafiltering, muennighoff2025scalingdataconstrainedlanguagemodels}, removing \emph{all} data filtering would be an extreme intervention that uses data considered to be actively harmful \citep{raffel2023exploringlimitstransferlearning}. Our goal in this work is to take this extreme seriously and study the limits of (low-quality) data for transformer pretraining. 

We find evidence that rejects the hypothesis that data filtering is necessary, and that eventually, no existing data filter is likely to improve upon training directly on Common Crawl. In our experiments, we scale down both CC and its filtered versions to keep their relative sizes intact, and then scale computational resources for pretraining on these different datasets. Our two main levers to do so are scaling model size (which requires more compute per training step) and training steps (which eventually leads to epoching). When comparing the best achieved performance, regardless of computational cost, our main finding is that the full pool outperforms our selected filters. 

Our findings are robust as we scale our experiments by 2 orders of magnitude, and we find that we can continue to see the effects from our small pool experiments as long as the models are sufficiently large. Furthermore, we find a predictable relationship between pool size, training steps, and model size which enables us to build scaling laws that predict how much compute is needed for no filter to be optimal for a particular pool size. Using this, we find that the 240 trillion token Common Crawl pool from DCLM-Pool may become optimal as soon as 1e+30 FLOPs.

These initial findings lead us to study the robustness of pretraining to ``junk'' data. Surprisingly, sufficiently large models are highly robust to irrelevant or junk data and can extract useful information even from highly noisy data. We test this using randomly generated strings and documents with shuffled word orders. While performance degrades at low compute budgets, sufficiently trained large models close the gap. Remarkably, these models even benefit from shuffled-word documents, despite only the unigram distribution of the documents remaining intact.

Overall, our experiments suggest that sufficiently large models that are trained for sufficiently long can benefit from the full CC dataset. While it is possible to construct harmful data, which could for example be non-factual content that looks identical to high-quality data, we do not find large amounts of this in CC. As a result, data filtering may suffer from the bitter lesson \citep{sutton2019bitter} in which human-designed filters that perform well at the small scale are eventually replaced by simple, no-filter approaches that scale more gracefully with compute. 

We structure the paper as follows. In Section~\ref{sec:preliminaries}, we provide the basic experimental setup, followed by experiments on filtering in Section~\ref{sec:filtering}. We then move to adding data to our CC pool in Section~\ref{sec:injection}, and scaling the pool size in Section~\ref{sec:scaling}. We finish with edge cases in Section~\ref{sec:degradation} and a theoretical model in Section~\ref{sec:theory} to provide a post-hoc explanation of the observed phenomena. 

\subsection{Related Work}

\textbf{Data-constrained pretraining.}
Several prior works consider the data-constrained pretraining regime.
\citet{muennighoff2025scalingdataconstrainedlanguagemodels} derive scaling laws that factor data repetition into the original Chinchilla scaling laws, finding diminishing returns after around 4 epochs on the data and that adding code
data and using looser perplexity-based filters mitigates data scarcity.
However, the authors recommend filtering ``noisy datasets'' and train on
subsets of C4 \citep{raffel2023exploringlimitstransferlearning}, while
the current work directly trains on (parsed) Common Crawl and finds evidence
in support of no filtering.
\citet{kim2025pretraininginfinitecompute} study the question of
algorithmic improvements in a data-constrained but compute-unlimited
setting. We share a similar experimental setup (where we take subsets of a dataset, scale compute on this subset, and then scale the subset size) but differ in the object of analysis (dataset filtering). 

\textbf{Loose data filters.}
The closest work to ours is 
\citet{goyal2024scalinglawsdatafiltering},  who argue that filter thresholds should depend on the compute
budget, showing evidence for vision-language models. 
They derive a scaling law to predict the filtering threshold as a
function of compute budget, and conclude that ``less aggressive
filtering is best'' with ``large compute'' but do not identify the parameter scaling interactions that are critical to our work, and do not show our main findings that for language models, no filter can be the best filter.
\citet{fang2025datasetsdocumentsrepetitionspracticalities} tackle a related
question by artificially repeating ``high-quality''
data to match the scale of loosely filtered data.
They find that the former can outperform the latter in low-compute regimes,
but the high compute regime studied in this work remains fully speculative in their work.
Finally, \citet{gao2021empiricalexplorationqualityfiltering} finds that
filtering aggressively can hurt performance, speculating that this
follows from Goodhart's law [\citeyear{Goodhart1984}], and
\citet{saada2025dataqualityillusionrethinkingclassifierbased} find that
filtering with a quality classifier may improve downstream benchmarks
but not validation losses on ``high-quality'' data.

On the theoretical side, \citet{cheng2024labelershavecloserlook} develop
theoretical models of the data cleaning process, arguing that
given models that have enough fidelity to model noisy data generation
schemes, it is better to not clean data, while cleaning data can
yield more robust learning when models are not perfect.
This prediction dovetails with our subsequent findings.

\textbf{Low quality data.}
Recent works' exploration of the impact of low-quality or intentionally
degraded data on model performance motivates our experiments in
Section~\ref{sec:injection}.
\citet{allenzhu2024physicslanguagemodels33} find that ``junk data''
significantly reduces knowledge capacity in a synthetic data setting, which
aligns with our findings on sufficient model sizes.
Counterintuitively, \citet{li2025baddataleadsgood} argue that pretraining on
toxic data leads to better representations, which makes it easier to remove
toxic behavior during the post-training phase. Investigating the limits of
data structure, \citet{sinha2021maskedlanguagemodelingdistributional} train
on shuffled-word data similar to our shuffled-word experiments, arguing that
the success of masked language models is primarily due to modeling
``higher-order word co-occurrence statistics''. Finally,
\citet{ru2025reallyfilterrandomnoise} train models on randomly generated
integers similar to our randomly generated text in
Section~\ref{sec:injection}, and notice only a small performance drop.

\section{Preliminaries}\label{sec:preliminaries}

We begin with our problem setup. Our goal is to measure the value of a dataset in terms of best possible performance, regardless of computational cost, on metrics of interest such as perplexity and downstream benchmarks. More formally, for a training algorithm $\mathcal{A}$ which accepts as arguments a dataset $D$ of any size, parameter count $M$, and training steps $N$, and outputs a model $\theta\in \Theta$ to be evaluated at a loss $\ell \colon \Theta \to \Rset$, our goal is to find the best achievable performance
\begin{align}\label{def:vod}
\mathcal{L}^\star(D) := \min_{M, N} \; \ell(\mathcal{A}(D, M, N)),    
\end{align}
as a function of the pretraining data. Our formulation has an unconstrained minimum over parameter count $M$ and training steps $N$ in an attempt to extract all the ``juice'' out of a dataset, no matter its size. Empirically, we compute this minimum by varying $M$ and $N$ over several orders of magnitude until either performance improvements start to plateau or we run out of compute.

Since we do not have the compute budget to train on all of Common Crawl (let alone perform multiple epochs), our experiments are structured around randomly sampled subsets. Let $D_{cc}$ be the entire CC, $D_{cc, m} \subseteq D_{cc}$ be a randomly sampled subset of $m$ tokens, and $f(D_{cc, m}) \subseteq D_{cc, m}$ be a filtered variant of the subset. In Section~\ref{sec:filtering}, we compare $\mathcal{L}^\star(D_{cc, m})$ and $\mathcal{L}^\star(f(D_{cc, m}))$ for standard filtering functions $f$ such as DCLM-Baseline and RefinedWeb and our smallest subset size $m$, to test if the commonly removed documents $D_{cc, m} \setminus f(D_{cc, m})$ are indeed helpful for improving performance. In Section~\ref{sec:injection}, we test model robustness by injecting various ``junk data'' $J$ to form $D_{cc, m} \cup J$, challenging the hypothesis that $\mathcal{L}^\star(D_{cc, m}) < \mathcal{L}^\star(D_{cc, m}\cup J)$ holds.  

Our smaller scale experiments implicitly assume that the better of $\mathcal{L}^\star(f(D_{cc, m}))$ and $\mathcal{L}^\star(D_{cc, m})$ does not change (or at least changes predictably) with $m$, which allows us to scale down and study the function $\mathcal{L}^\star$ at reasonable compute budgets. To investigate whether this is indeed the case, and understand how performance changes as a function of $m$, $M$, and $N$, we additionally scale over the pool size $m$ in Section~\ref{sec:scaling}.

\subsection{Experiment details}

We use the version of Common Crawl provided by \citet{li2025datacomplmsearchgenerationtraining} in their DCLM-Pool dataset, which is all of CC before $2023$ with text extracted from HTML via $\texttt{resiliparse}$ \citep{bevendorff:2018}. This dataset is $240$ trillion GPT-NeoX \citep{black-etal-2022-gpt} tokens and our randomly sampled subsets range from about $670$ million to $10$ billion tokens. When filtering, we use the code provided by \citet{li2025datacomplmsearchgenerationtraining}. We do not use any specialized data curricula or data weights. 

Our models are Llama-style dense transformers ranging from $15$ million to $7$ billion parameters, trained with the Meta Lingua code repository \citep{meta_lingua}. For each of the models, we tune the training step count and weight decay, following prior studies to increase repeatability of the data \citep{fang2025datasetsdocumentsrepetitionspracticalities, kim2025pretraininginfinitecompute}. As is standard, we set the learning rate to decay with model size \citep{brown2020languagemodelsfewshotlearners,kaplan2020scalinglawsneurallanguage}, with an initial tuning stage to determine the decay. We release our configuration files on GitHub.\footnote{\url{https://github.com/chrismohrii/bitter-lesson-data-filtering}}

Our main metrics of interest are the loss (negative log-likelihood) on various datasets, since this is known to correlate with downstream performance and provides smoother measurements than common question-answering benchmarks (likely due to their small size). These datasets are the English portion of C4 \citep{raffel2023exploringlimitstransferlearning}, Fineweb-Edu \citep{penedo2024finewebdatasetsdecantingweb}, which is a pretraining dataset targeting educational texts, and Cosmopedia \citep{benallal2024cosmopedia}, a dataset of synthetically-generated texts. We primarily plot the average loss across these three, but the trends are the same for each individually as well. We also provide results on common benchmarks such as ARC-Easy \citep{clark2018thinksolvedquestionanswering} and PIQA \citep{bisk2019PIQAreasoningphysicalcommonsense} in Appendix~\ref{app:benchmarks}. Since our experiments use pool sizes of only up to 10B tokens, we do not expect to suffer from test set contamination. 

\section{Data Filtering}\label{sec:filtering}

In this section, we test the hypothesis that standard filtered versions of CC achieve a lower loss than the unfiltered CC. Returning to our formulation in~(\ref{def:vod}), when $D_{cc, m}$ is an $m$-token subset of CC and $f$ is a filtering function, we are interested in the best of  $\mathcal{L}^\star(D_{cc, m})$ and $\mathcal{L}^\star(f(D_{cc, m}))$. While we evaluate a representative set of standard and relaxed filters, an exhaustive search over the exponential space of subsets is computationally intractable. Our objective is instead to benchmark open curation strategies against the pool and identify if models are able to extract signal from ``low-quality'' data.

We focus on our smallest CC pool size (about $670$ million tokens) where ablations are the cheapest, and curate five filtered versions of this pool by applying the filters described below, all of which are used in \citet{li2025datacomplmsearchgenerationtraining}. The first three are individual filters applied in the initial ``heuristic cleaning'' stage of DCLM-Baseline, and ablating them alone gives us pretraining datasets that are larger and more loosely filtered than standard. The fourth gives the end result of the ``heuristic cleaning'' stage, and the last gives the result of the full filtering pipeline. 
\begin{figure}[t]
    \centering
    \includegraphics[width=1\linewidth]{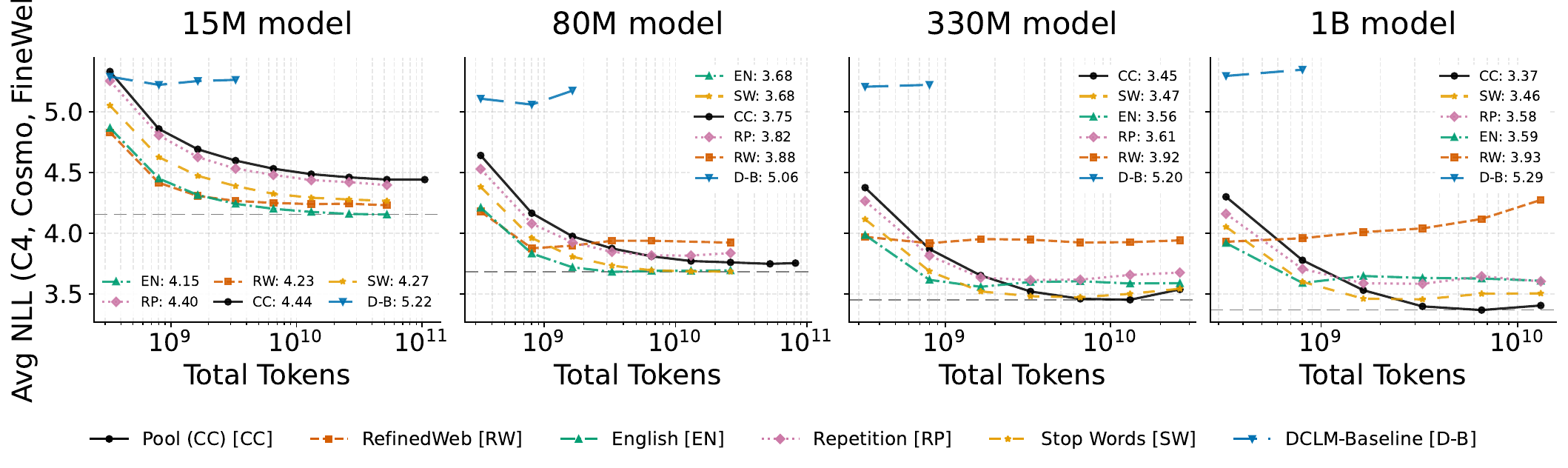}
    \caption{Comparison of models on 670M token CC pool and five filtered subsets. For sufficiently large models (330M+), the unfiltered pool (black) outperforms all five filters (colors) after sufficiently many optimization steps (x-axis, tokens under multiple epochs).}
    \label{fig:filters-600m-pool}
\end{figure}

\begin{wrapfigure}{tr}{0.4\textwidth} 
  \centering
  \includegraphics[width=0.38\textwidth]{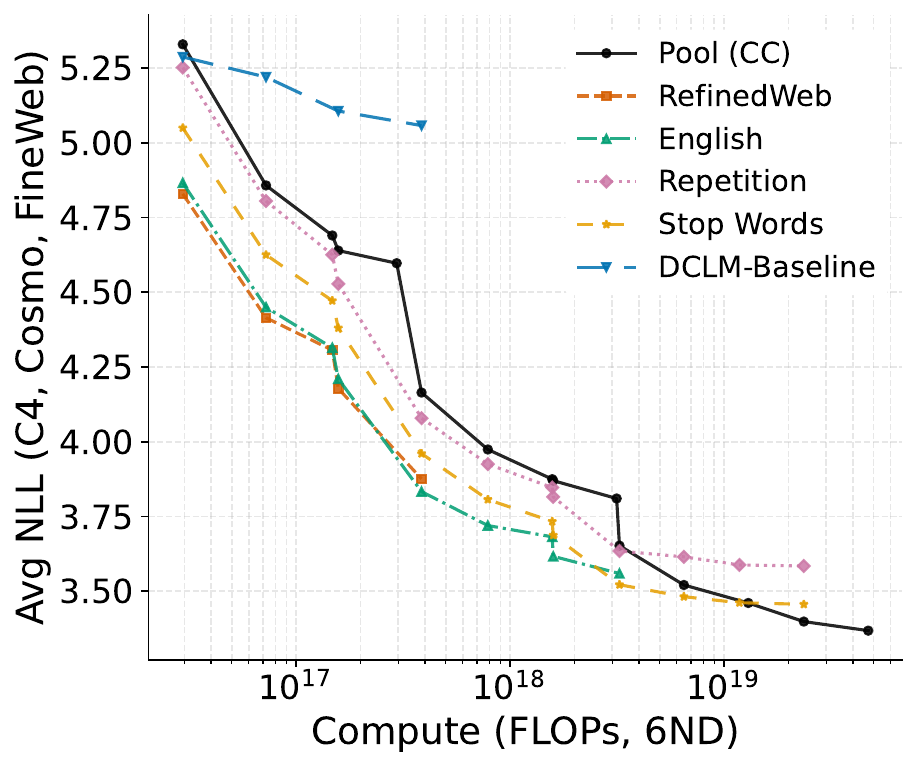}
  \caption{Pareto frontier of Figure~\ref{fig:filters-600m-pool} showing that in high-compute regimes, pool becomes optimal.}
  \label{fig:pareto-670m-pool}
\end{wrapfigure}

\textbf{English filter.} This filter first obtains an English score for a document using a fastText classifier \citep{joulin2016bagtricksefficienttext} and then applies a threshold to this score. According to our tokenizer, $28.2 \%$ of the data is left after applying this filter.

\textbf{Repetition filter.} This filter originates from the data curation stage of the Gopher model, with the motivation that ``excessive repetition is often linked with uninformative content'' \citep{rae2022scalinglanguagemodelsmethods}. It splits documents into segments of various granularities, such as lines, paragraphs, or n-grams, and applies a threshold on the duplicate fraction of these segments. According to our tokenizer, $45.3 \%$ of the data is left after applying this filter.

\textbf{Stop word filter.} This filter ensures that a document contains at least 2 occurrences of English stop words from the following list: ``the'', ``be'', ``to'', ``of'', ``and'', ``that'', ``have'', and ``with''. According to our tokenizer, $50.4 \%$ of the data is left after applying this filter.

\textbf{RefinedWeb.} This consists of the filters above along with other similar filters, in an attempt to reproduce the RefinedWeb dataset \citep{penedo2023refinedweb}. According to our tokenizer, $13 \%$ of the data is left after applying this filter.

\textbf{DCLM-Baseline.} This dataset applies deduplication and quality-based filtering with a fastText classifier to RefinedWeb. According to our tokenizer, $2.1\%$ of the original pool data is left after applying this filter. We address questions of severe data scarcity in Appendix~\ref{app:benchmarks}.

In Figure~\ref{fig:filters-600m-pool}, we show the average loss for each dataset as compute is varied with both model size and training steps. Each point consists of a separate training run, with its own warmup and cosine decay learning rate schedule.  Overall, the pool (CC) reaches the best loss of $3.37$ on the 1B model, and its loss has not visibly plateaued from scaling model size. Outperforming the filtered datasets requires both a sufficiently large model \textit{and} a sufficiently large training step count. While we have not trained the $15$M model until the loss starts to increase again because the loss continues to decrease even at a training budget of $100$B total tokens, it does not appear as though the pool will ever outperform any of the first four filtered datasets. As we transition to the larger models, we observe crossing points on the loss curves between the pool and filtered versions, and these crossing points appear earlier as model size increases. 

In Figure~\ref{fig:pareto-670m-pool}, we take the same runs from Figure~\ref{fig:filters-600m-pool} and derive a compute-performance Pareto frontier. We calculate the compute for a run with the standard $6NM$ approximation \citep{kaplan2020scalinglawsneurallanguage}, where $N$ is the number of total training tokens and $M$ is the number of model parameters. As compute is increased, the pool transitions from the worst-performing dataset to the best. Perhaps surprisingly, not all datasets enjoy a point on the overall Pareto frontier: at every given compute level, there are at least two better-performing datasets than the repetition filtered dataset.

Overall, these experiments suggest that pretraining is surprisingly resilient. Even at our scale, we see that the pool eventually beats the performance of all the filtered variants. This can be counterintuitive, since we might expect some junk data to hurt model performance. To further explore this phenomenon, we create artificial low quality data to probe the limits of pretraining robustness in the next section.

\section{Data Injection}\label{sec:injection}

We now test the limits of model robustness by deliberately injecting low-quality data. We investigate the hypothesis that the best achievable performance strictly degrades when curated ``junk'' distributions are added to the pretraining pool. More formally, if $D_{cc, m}$ is a subset of CC and $J$ represents the injected low-quality dataset, we are interested in the best of $\mathcal{L}^\star(D_{cc, m})$ and $\mathcal{L}^\star(D_{cc, m} \cup J)$. Our first variant of $J$ is designed to be devoid of any useful signal, and the second is designed to have some useful signal but of extremely low quality (Examples in Figure~\ref{fig:junk_examples}).

\textbf{Randomly generated strings.} We define a vocabulary of 10,000 words by uniformly sampling 3 to 8 characters from the lowercase English alphabet (a-z). We then sample uniformly from these words and concatenate them with a space character to form documents.

\textbf{Additional shuffled pool documents.} We take additional CC documents that are not included in our CC subset and randomly shuffle the order of the words in each document.

\begin{figure}[t]
    \centering
    \includegraphics[width=1\linewidth]{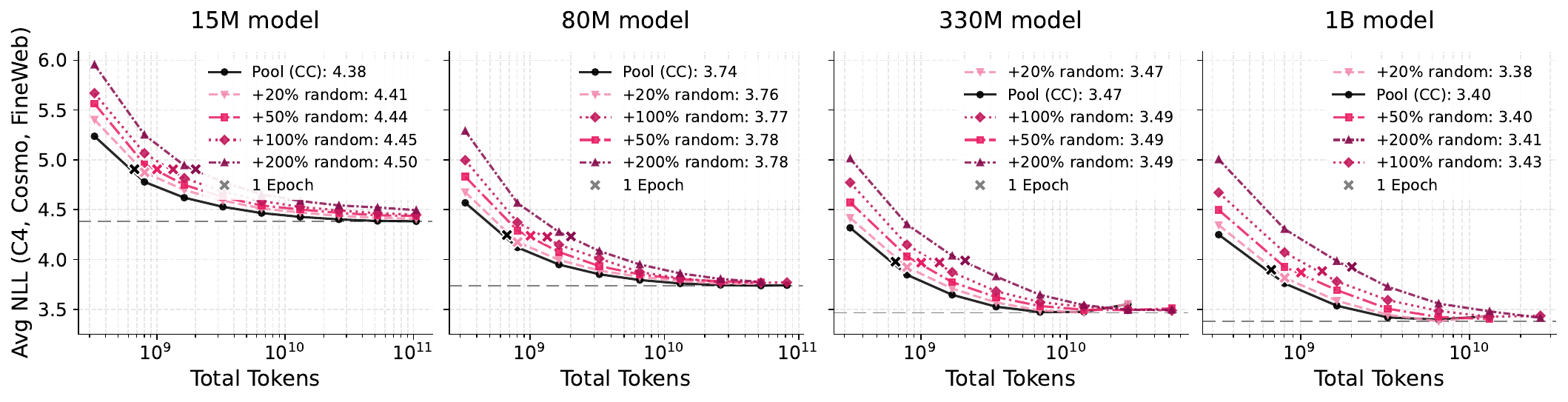}
    \includegraphics[width=1\linewidth]{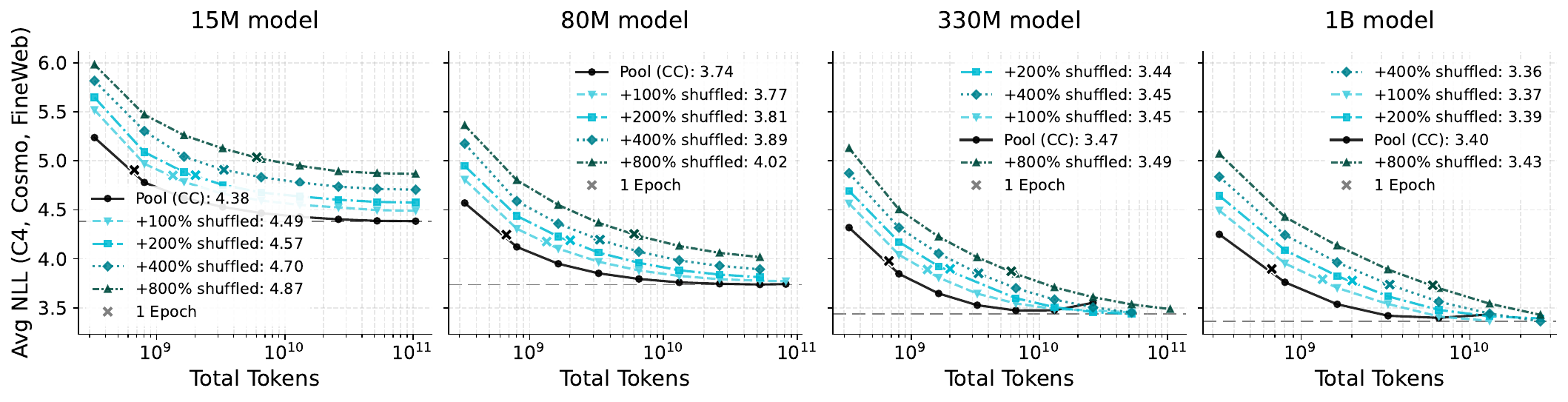}
    \caption{670M-token CC pool versus junk-injected versions. Plots show a surprising robustness to random data (top) for large models with consistent gains from low-quality (shuffled) data (bottom) with sufficiently many epoched training steps (x-axis). }
    \label{fig:injection-600m-pool}
\end{figure}

In Figure~\ref{fig:injection-600m-pool}, we provide the comparison of the two new datasets alongside the CC pool when varying model size and training step count. We have included varying amounts of injected junk data, up to 8 times the pool size in the shuffled words case, leaving only about 10\% of untouched CC documents. In both cases, it is immediate that the injected data has not completely reduced model performance to random performance, which would result in a cross-entropy loss or negative log-likelihood of $-\log(1/V)$ where $V$ is the vocabulary size, giving approximately 10.8 with our tokenizer. 

For both dataset variants, a sufficiently large model is required to match the pool performance. With the 15M model, there is a separation in the loss curves, regardless of the ratio of injected documents. As we transition to larger models, this gap closes. On the 330M model, we even see that all of the shuffled datasets---except the +800\% shuffled dataset---surpass the pool performance after around 11B training tokens. We have not trained the +800\% shuffled dataset past 100B tokens, but we expect it will also surpass pool performance since its loss has not visibly plateaued. We also expect it to cross this threshold even earlier on the larger 1B model because of its faster-decreasing loss. In the case of the randomly generated strings, the random datasets appear to more closely match the performance of the pool, but overall, the gaps are still closing with model size.

Our intuition for the differences between these datasets is that the shuffled words are more ``confusing'' for a smaller model, whereas the randomly generated strings are more clearly drawn from a different distribution. As we scale model size, and thus perhaps the ability to differentiate between the two distributions, there is more signal to extract from the shuffled data as it contains additional unseen pool documents with the unigram distribution intact. If, for example, we shuffled the sentence ``The capital of France is Paris'', we would still see ``France'' and ``Paris'' together, helping the model understand that there may be some connection between the two. We attribute the improved performance with +20\% random to either a potential regularization effect or an unintended similarity to natural text, which generally features words of similar lengths separated by space characters.

\begin{figure}[t]
    \centering
    \begin{subfigure}[t]{0.49\textwidth}
        \fbox{
            \begin{minipage}{0.9\linewidth}
                this RC [English]WLtoys topics cannot You cannot and Quadcopter Instruction and Replies post attachments in your other \ignore{Quadcopter cannot Topic You You Instruction forum download RC in reply this posts post posts edit permissions your download Forum forum in Last New You in [English]WLtoys New this topics to Manual }$\ldots$
            \end{minipage}
        }
        \label{fig:sub_random_example}
    \end{subfigure}
    \hfill 
    \begin{subfigure}[t]{0.49\textwidth}
        \fbox{
            \begin{minipage}{0.9\linewidth}
                htb hqovl bwdws wesqae wcb xkk xhkqfm jhvbvutr nqxm ykzpnklm trgikh nymn dcncwn osyrr zpvrrly yhdsrr nyvo ynx \ignore{kodgmkf owdrj kwewce hthz tdzejcd hgcbeom wok veyrr tloheva cljdboy tlnxkaq isgu begzsn otna xlnmbfrp ocepbeav rapjn impg fohe dwbclg luau pha akje cneyh gsqwbex lphxv khg ceyjmut vlbitr bualnom cmv wfr otij slzizkk kcx hutzsd snxsivl axtppj rptwuwsc bgzarjql ngvyz vxzlq jfbfnp turqtrps dsqnf kveuulgf wndb cwkigk bjhk tesdpl klyibq vjpz gswhuthu} $\ldots$
            \end{minipage}
        }
        \label{fig:sub_shuffled_example}
    \end{subfigure}
    \caption{Examples of ``low-quality'' documents injected into CC pool.  \textbf{Left}: documents with shuffled word order. \textbf{Right}: documents with randomly generated strings.}
    \label{fig:junk_examples}
\end{figure}

\section{Scaling Pool Size}\label{sec:scaling}

Do our experiments have implications for large-scale pretraining where the pool is all of CC? While suggestive, our 670M pool sample is quite far from the available internet stock of 200-500 trillion tokens \citep{villalobos2024rundatalimitsllm}, and scale effects could significantly change our conclusions.

To address these concerns, we turn to scaling studies that show our effects are consistent across scale by varying our pool size and model sizes across $2$ orders of magnitude, and build up to a prediction of the compute threshold where the CC pool in DCLM-Pool (240T tokens) outperforms RefinedWeb. Due to the computational costs of these runs, we focus solely on the comparison between CC and $f=$ RefinedWeb, with the goal of making a prediction on the better of $\mathcal{L}^\star(D_{cc})$ and $\mathcal{L}^\star(f(D_{cc}))$.

\begin{figure}[ht]
    \centering
    \includegraphics[width=1\linewidth]{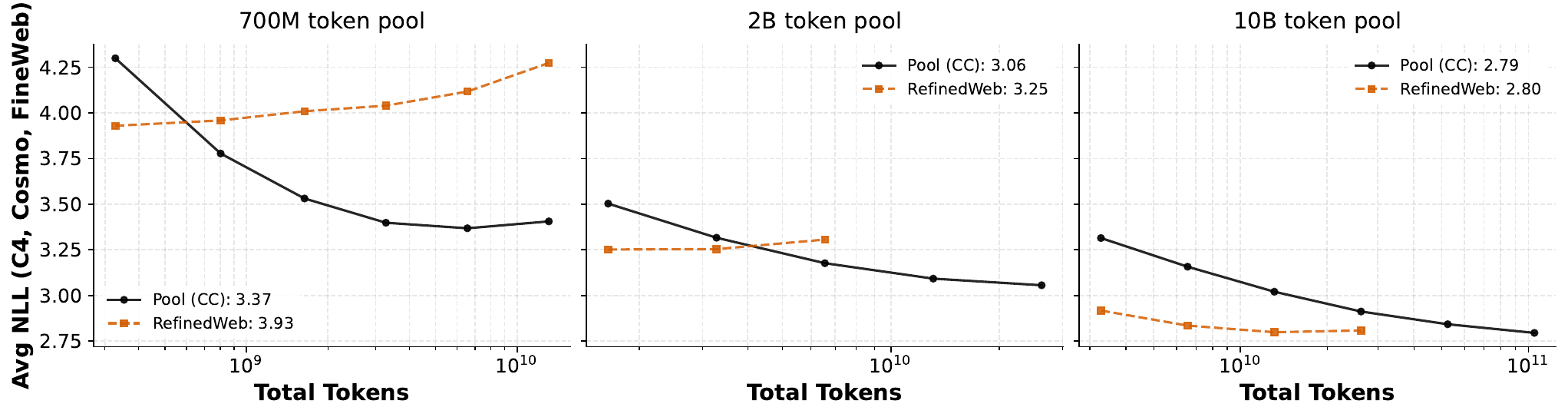}
    \includegraphics[width=1\linewidth]{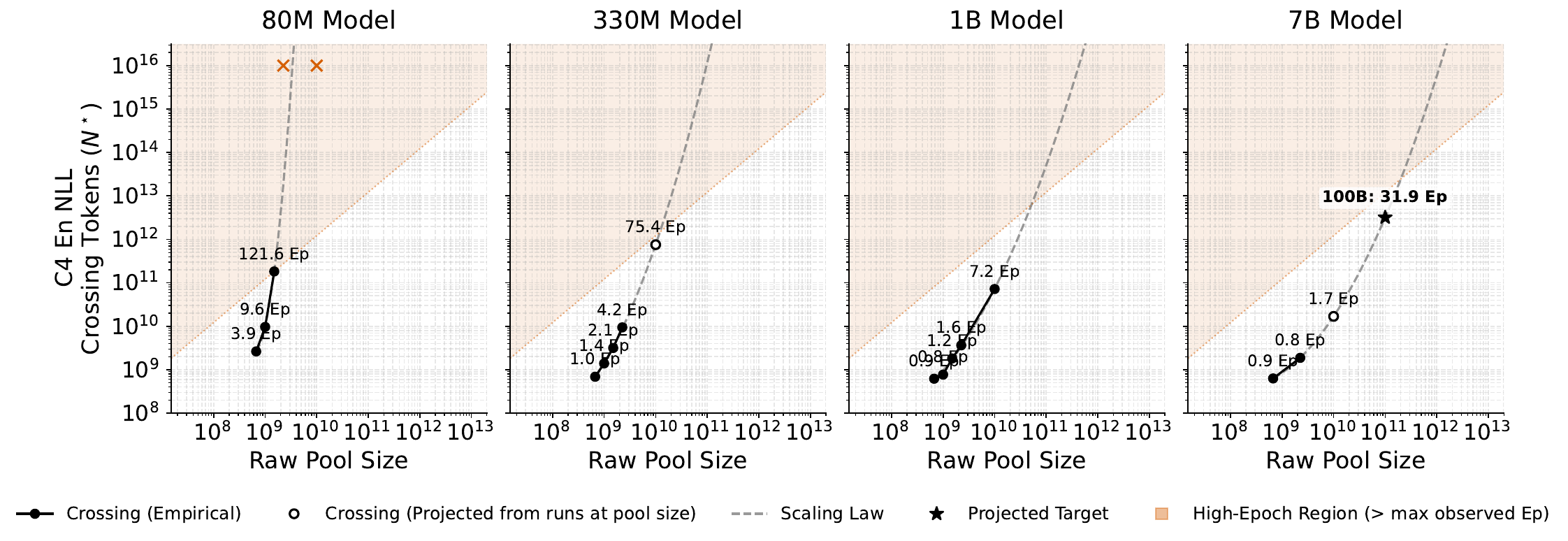}
    \caption{\textbf{Top:} $1$B model performance as we vary the pool size; the total needed steps for pool to outperform RefinedWeb grows rapidly.
\textbf{Bottom:} Crossing point as a function of pool size for various model sizes. Markers each represent a crossing point (e.g. top panel), with text showing the epoch count. Epochs above the largest observed crossing point (121.6 epochs) are shaded to indicate unreliability at extreme epoch counts. Dashed lines show second-order polynomial fits used to interpolate data and show growth trends.}
    \label{fig:crossing-points}
\end{figure}

Understanding how pool size affects performance requires us to map out the joint space of pool size $m$, model parameters $M$ and step count $N$. As a simplifying first step, we represent step count as a function of the other two variables,
\[N
^\star(M, m) := \min \left\{ N \colon \ell(\mathcal{A}(D_{cc, m}, M, N)) < \min_{N'} \ell(\mathcal{A}(f(D_{cc, m}), M, N'))\right\},\]
where we have taken the minimal winning $N$ (if one exists) as the output of the function. Given our intuition and experimental evidence that performance improves with larger models when sweeping over step count (see Figures~\ref{fig:filters-600m-pool} and~\ref{fig:injection-600m-pool}), this serves as a succinct representation of our 3 variable space.

Our step count function $N^\star$ has predictable behavior in both of its arguments. When we fix $M=1$B and increase the pool, we make two important observations. First, we see that the point at which the pool performance becomes better than RefinedWeb ($N^\star$) grows rapidly (top half of Figure~\ref{fig:crossing-points}), and the precise quantitative rate of growth is super-linear (roughly 10 epochs are needed for the 10B-token pool, compared to roughly three epochs for the 2B-token pool and one epoch for the 670M-token pool). Our second observation is that the validation losses are nonmonotone even with tuned weight decay regularization, suggesting that in extreme epochs (100+), the two may cease to cross.

\begin{wrapfigure}{tr}{0.4\textwidth} 
  \centering
  \includegraphics[width=0.4\textwidth]{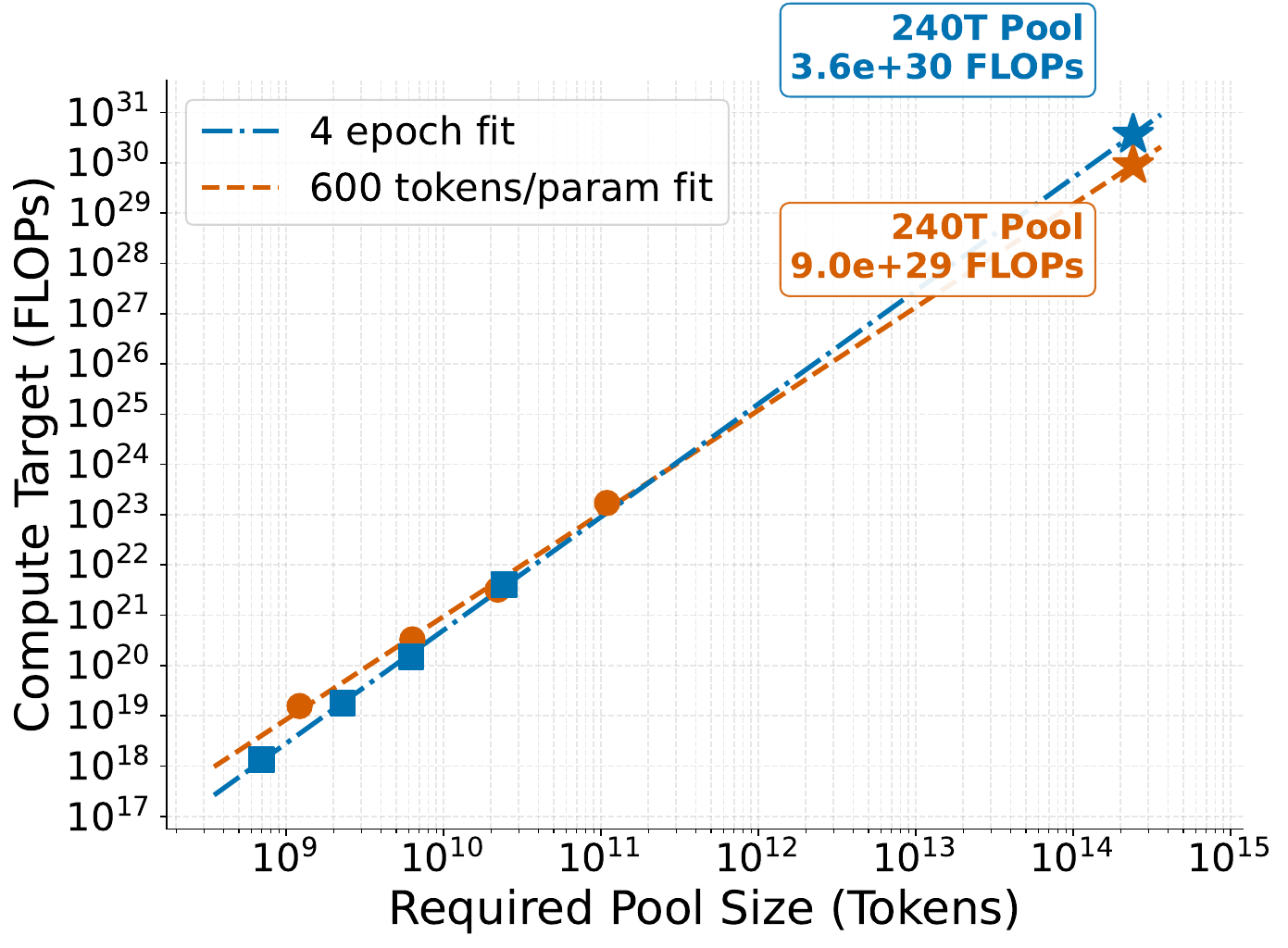}
  \caption{Scaling laws for optimality of no data filtering. Two scaling laws with token-per-parameter scaling (in orange) and epoch constraints (in blue) both give highly linear scaling and predict similar budgets (1e+30 FLOPs).}
  \label{fig:extrapolation}
\end{wrapfigure}

We now also vary model size $M$ to understand the joint scaling behavior as model size grows with pool size. Figure~\ref{fig:crossing-points} shows a sweep over $N^\star(M,m)$ with each panel varying $M$ and the x-axis varying $m$. On the leftmost plot with the 80M model, we can clearly see that crossing points cease to exist, even across our evaluated pool sizes: while there is a crossing point for the smallest 670M-token pool, there is no longer a crossing point on the largest 10B-token pool as indicated by the dark orange marker. As high-epoch regimes can become nonmonotone, we mark those regions in orange in the plot to indicate that they are unlikely to have any crossing points. As we scale up $M$, however, we see that the epoch counts needed for the pool to win rapidly decrease as a function of model size.

With these observations and our experimental measurements of $N^\star$, we can answer our question of what happens when we scale our pool sizes to the current CC pool size (240T tokens in DCLM-Pool). Are compute levels in the near future likely to reach a point where the entire CC pool is better than RefinedWeb? We follow a simple procedure to build a compute scaling law on top of our $N^\star$ function (Figure~\ref{fig:crossing-points}), fitting two types of scaling laws to be robust to misspecification. In our first approach, we start by specifying a token-to-non-embedding-parameter ratio (600:1, following DeepSeek V4). For each model size, this ratio immediately specifies the number of training steps $(N^\star)$ as well as the compute level $(C=6MN^\star)$. We can then estimate the pool size corresponding to this $N^\star$ for each model (using a fitted quadratic to interpolate among our observed data points as described in Appendix~\ref{subsec:scaling_fits}) and build a scaling law against $C$. In our second approach, we instead specify an epoch count (4, based on \citet{muennighoff2025scalingdataconstrainedlanguagemodels}). The epoch count specifies a linear constraint which intersects $N^\star$ for each model at a single point (cf. the orange 120-epoch line in Figure~\ref{fig:crossing-points}). This point specifies the pool size and compute level, which we can then also use to build a scaling law. 

In contrast to the training steps $N^\star$, our compute scaling laws are highly linear (Figure~\ref{fig:extrapolation}), with $R^2$ above $0.99$, and both give similar predictions, near 1e+30 FLOPs for the crossing point. This compute level is quite high, with the best current estimates of frontier pretraining compute near 5e+26 \citep{grok4_modelcard}, but this is far from an outlandish amount of near-future compute, with existing forecasts predicting 1e+29 FLOP training runs by 2030 \citep{epochepri2025aipower}.

\section{Model Degradation}\label{sec:degradation}

In all of our experiments so far, we have seen that regardless of the distribution, more data helps if we are free to train a sufficiently large model for sufficiently long. We should not expect this to be a universal property in machine learning, as a large body of research has been dedicated to the problem of domain adaptation and learning under distribution shift \citep{mansour2023domainadaptationlearningbounds, pmlr-v206-awasthi23b}. Instead, we hypothesize that language models are highly resistant to covariate shifts, and it is ``incorrectly labeled'' data or data with shifts in the conditional distribution from a target metric that can be detrimental. For example, we expect that a model trained on sufficient instances of ``The capital of France is Copenhagen'' will learn the wrong capital of France.

\begin{table}[t]
    \centering
    \caption{Average GPT5-mini judgements on keyword-matched CC data for select MMLU categories.}
    \label{tab:mmlu_averages}
    \vspace{0.1in} 
    \begin{tabular}{lcccc}
        \toprule
        \textbf{Dataset} & \textbf{Support} & \textbf{Refute} & \textbf{Related} & \textbf{Unrelated} \\
        \midrule
        MMLU/world\_religions & 5.89 & 0.00 & 13.22 & 7.50 \\
        MMLU/astronomy & 2.03 & 0.14 & 10.14 & 17.41 \\
        MMLU/college\_biology & 2.67 & 0.17 & 11.07 & 13.40 \\
        MMLU/medical\_genetics & 2.80 & 0.23 & 14.30 & 11.23 \\
        \bottomrule
    \end{tabular}
\end{table}


\begin{wrapfigure}{r}{0.4\textwidth}
    \centering
    \includegraphics[width=\linewidth]{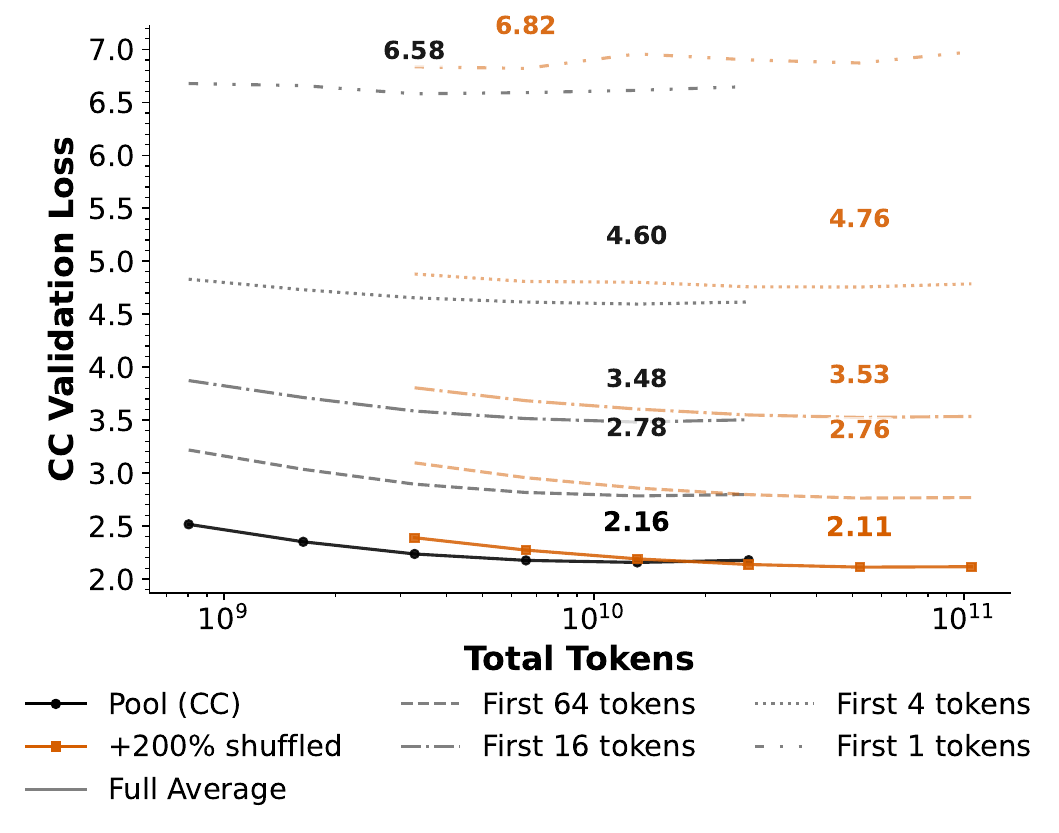}
    \caption{330M model: loss of $670$M pool subset versus $+200\%$ dataset. }
    \label{fig:first-token-losses}
\end{wrapfigure}

While CC is too large to exhaustively search through and contains non-factual content such as conspiracy theories, we argue that such actively harmful content is relatively low frequency. We provide a very brief study to support this with a corpus analysis of MMLU-related documents in CC \citep{hendrycks2021measuringmassivemultitasklanguage}. We first match keywords, and then we prompt GPT5-mini to classify whether the document supports, refutes, is related, or is unrelated to the question and answer. We target MMLU subjects such as world religions, where there are very rare keywords. We present our analysis in Table~\ref{tab:mmlu_averages}. While our search did find mostly unrelated or related but neither supporting nor refuting documents, the average number of documents in support is at least an order of magnitude larger than refuting. In Appendix~\ref{app:theory}, we develop some theory to provide an analysis of when filtering should help, in terms of how factual or correctly labeled a dataset is. 

We now move to a case of distribution shift from our experiments with shuffled word order documents in Section~\ref{sec:injection}. Our metrics were the average validation loss across the entire sequence, but we may expect to suffer from a distribution shift with the loss on the initial tokens in a document, because we changed the distribution from the natural distribution of first tokens that appear in CC. In the case of predicting the very first token, it is impossible to detect whether a document is shuffled by having access only to the empty prefix.

In Figure~\ref{fig:first-token-losses}, we compare the average CC validation loss for CC and the $+ 200\%$ shuffled dataset when we look at the loss on initial segments of the document. 
As we transition from the full average to the loss on only the very first token, the $+200\%$ shuffled dataset loses its advantage over the pool. We do not expect this behavior to change with larger models. However, as most use cases of language models involve more than just a few tokens, we do not anticipate that this is a meaningful degradation.

\section{Theoretical Models}
\label{sec:theory}

We might ask whether the results we have identified are predictable: ought
we expect them?
We present two theoretical models, one here and one
in the appendices, that exhibit the behaviors we see,
suggesting the types of behavior we identify might hold more broadly.



Heuristically, we might hypothesize that once a (transformer)
model is large enough, it can pass ``bad'' data through components
that do not interact with components representing ``good'' data, and
when a model is not large enough, this cannot occur.
Our experiments are consistent with this: large models absorb unfiltered
data without penalty, while smaller models cannot.
In low-rank matrix factorization---the simplest 1 hidden layer (linear)
neural network---we see exactly this behavior at the population level.


To make this more precise, consider predicting vector-valued outputs $y$
(tokens) using a rank $r$ linear transformation of an input $x$.
Assume the pairs $(x, y) \in \Rset^d \times \Rset^m$ come from one of $k$
tasks, where task $i$ occurs with probability $p_i > 0$ and generates $Y =
u_{\star,i}\, v_{\star,i}^\top X_i + \xi$ for independent mean-zero noise $\xi \in
\Rset^m$, where $\Sigma_i = \E[X_i X_i^\top]$ satisfy $\tr(\Sigma_i
\Sigma_{i'}) = 0$ for $i \neq i'$, so that tasks have orthogonal inputs:
one may exactly separate them.
The next proposition, whose proof is in Appendix~\ref{app:theory},
follows.

\begin{proposition}[Rank Necessity under Orthogonal Inputs]
  \label{thm:rank-necessity}
  Let the conditions above hold and $M_{\star,i} = u_{\star,i}\, v_{\star,i}^\top$,
  and define $M_\star = \sum_{i=1}^k M_{\star,i}$
  and $\Sigma = \sum_{i=1}^k p_i\Sigma_i$.
  If $\sigma_1 \ge \cdots \ge
  \sigma_\rho > 0$ are the $\rho \leq k$ positive singular values of
  $M_\star\Sigma^{1/2}$, then for any model rank $r$
  \[
  \min_{\substack{U \in \Rset^{m \times r}\\ V \in \Rset^{d \times r}}} \E\!\big[\|Y - UV^\top X\|^2\big] = \sum_{j=r+1}^{\rho} \sigma_j^2 + \E\!\big[\|\xi\|^2\big],
  \]
  where the sum evaluates to $0$ if $r \ge \rho$.
\end{proposition}

The result makes clear that, given a large enough model rank $r$,
a matrix factorization can optimally represent the prediction problem
(so long as $r \ge k$).
On the other hand, without enough capacity
($r < \rho$),
model performance necessarily degrades with interference of the
tasks in $Y$-space, as the singular values of $M_\star \Sigma^{1/2}$ capture.
Moreover, at least at this population level, (regularized) gradient-based
methods are guaranteed to find the optimal matrices $U$ and $V$, because the
objective $\E[\|Y - UV^\top X\|^2]$ has no non-strict saddle points when $r
\ge k$~\citep{BaldiHo88, ZhuLiTaWa18}, and gradient descent converges to
local minimizers with probability 1~\citep{LeeSiJoRe16}.
In a fairly precise sense, then, this simple matrix factorization model
exhibits much of the behavior we see in experiments: with enough capacity,
noise (tasks) can be immediately absorbed, while smaller models suffer, and first-order methods are sufficient
for optimal fitting.

\section{Discussion}

While we have identified ways that scaling compute appears to make filtering immaterial, there are several limitations that lead to natural next steps for research in this direction.

\textbf{Deviations from vanilla pretraining.}  Our setting is restricted to pretraining on dense transformer models, without any data curricula, data weights, or post-training. There may be more unstable architectures such as Mixture of Experts models (MoEs), or phenomena in later stages of training, that require more careful choices with the pretraining data. Other changes, like pretraining on synthetic data, can also have an effect. If we view synthetic data as just augmenting ``high quality'' data and assume that “low quality” data does provide useful information for improving metrics, then including synthetic data may just increase the compute level for the crossing points as in Section~\ref{sec:scaling} by providing more effective tokens. However, if low-quality data mainly acts as a regularizer, then synthetic data may be strictly better. 

\textbf{Duplicate documents}. The expected fraction of duplicate documents increases with subset size. At our subset size, it is likely much smaller than the entire CC. We do not expect that our general conclusions would change, especially as we epoch the data, but this is a variable that likely does not play a large role in our experiments. 

\textbf{Compute.} The compute required for raw Common Crawl to outperform our tested filters is large, up to around $1e30$ FLOPs with our projections in Figure~\ref{fig:extrapolation}. When compute is a bottleneck, we expect filtering to still be important.

\textbf{AI-generated content.} We expect the fraction of AI-generated content in CC to increase, with likely a small amount in our pre-2023 DCLM-Pool dataset. It is unclear whether this will be detrimental.

\textbf{Factuality.} We have conducted an initial study into CC factuality or correctness with Table~\ref{tab:mmlu_averages}, but there are likely some rare edge cases where models trained on the full pool learn inaccuracies.

\newpage
\bibliographystyle{plainnat} 
\bibliography{neurips_2026}

\newpage
\appendix

\section{Experimental Details}

\textbf{Hyperparameters.} Across all experiments, we use a context length of $1024$ tokens, batch size of $2^{19} = 524,288$ tokens, and a $500$ training step warmup. We provide model-specific details in Table~\ref{tab:model_configs}.  All runs have a weight decay tuned in $[0.1, 0.5]$. The learning rates for the models were obtained with an initial tuning stage (and they also match the default learning rates for the 1B and 7B models in the Lingua repository). 

\textbf{Training details and compute.}
Throughout the plots (for example Figure~\ref{fig:filters-600m-pool}), we vary the training steps as powers of $2$. We evaluate a model $5$ times during training and report the best checkpoint (which is almost always the last one, except for rare cases when the training steps are very large compared to data size). All experiments were conducted on H200 GPUs. Each run used only data parallelism on a single 8-GPU node, except for the 7B model which also uses FSDP, varying from less than an hour to up to 2-3 days.  The combined cost of all our experiments exceeds 20,000 H200 GPU hours. 

\begin{table}[h]
\centering
\caption{Model architecture configurations. }
\label{tab:model_configs}
\begin{tabular}{lcccccc}
\toprule
Model & Hidden dim & Layers & Heads & Head dim & Learning rate & Weight decay \\
\midrule
15M   & 128  & 8  & 8  & 16  & $1\times10^{-2}$ & $[0.1, 0.5]$ \\
80M   & 512  & 8  & 8  & 64  & $5\times10^{-3}$ & $[0.1, 0.5]$ \\
330M  & 1024 & 18 & 16 & 64  & $5\times10^{-3}$ & $[0.1, 0.5]$ \\
1B    & 2048 & 17 & 16 & 128 & $5\times10^{-3}$ & $[0.1, 0.5]$ \\
7B    & 4096 & 32 & 32 & 128 & $1\times10^{-3}$ & $[0.1, 0.5]$ \\
\bottomrule
\end{tabular}
\end{table}

\subsection{Scaling law fits}\label{subsec:scaling_fits}

In several of our plots, we fit scaling laws to our empirically obtained measurements. 

\textbf{Figure~\ref{fig:crossing-points}.}
In the bottom half, we fit a second-degree polynomial to the log-log plot due to the super-linear and eventually infinite behavior. The hollow points on the plot are obtained from training runs at the given pool size, but with step counts prior to the crossing point. In those cases, we fit a power law to the (decaying) loss, and extrapolate the first training step or token count where the pool surpasses the best RefinedWeb loss (achieved or extrapolated). The cases where no crossing is ever predicted are marked with an orange ``x'' on the plot, and only appear on the 80M model size plot.

\textbf{Figure~\ref{fig:extrapolation}.} We use a standard power law, where the input is pool size and the output is the compute target. We use the number of non-embedding parameters as the model size when computing for example the 600 token/parameter ratio. 

\section{Additional experiments}\label{app:benchmarks}

In this section, we begin with downstream benchmark results to complement the validation loss metrics from the main text. We use PIQA, ARC-Easy, and SocialIQA \citep{sap2019socialiqacommonsensereasoningsocial} as these have easy enough questions to provide signal at our scale.

In Figure~\ref{fig:filters-600m-pool-benchmarks}, we provide plots analogous to those in Figure~\ref{fig:filters-600m-pool} but for the benchmarks, and Figure~\ref{fig:pareto-600m-pool-benchmarks} is similarly analogous to Figure~\ref{fig:pareto-670m-pool}.
We do the same for the injected datasets: Figure~\ref{fig:random-600m-pool-benchmarks} shows the datasets with random injection and Figure~\ref{fig:shuffled-600m-pool-benchmarks} shows the datasets with shuffled word order. These plots are in general much noisier than the perplexity-based ones in the main text, likely due to the relatively small number of questions in the benchmarks. However, the trends are roughly the same.

\begin{figure}[t]
    \centering
    \includegraphics[width=1\linewidth]{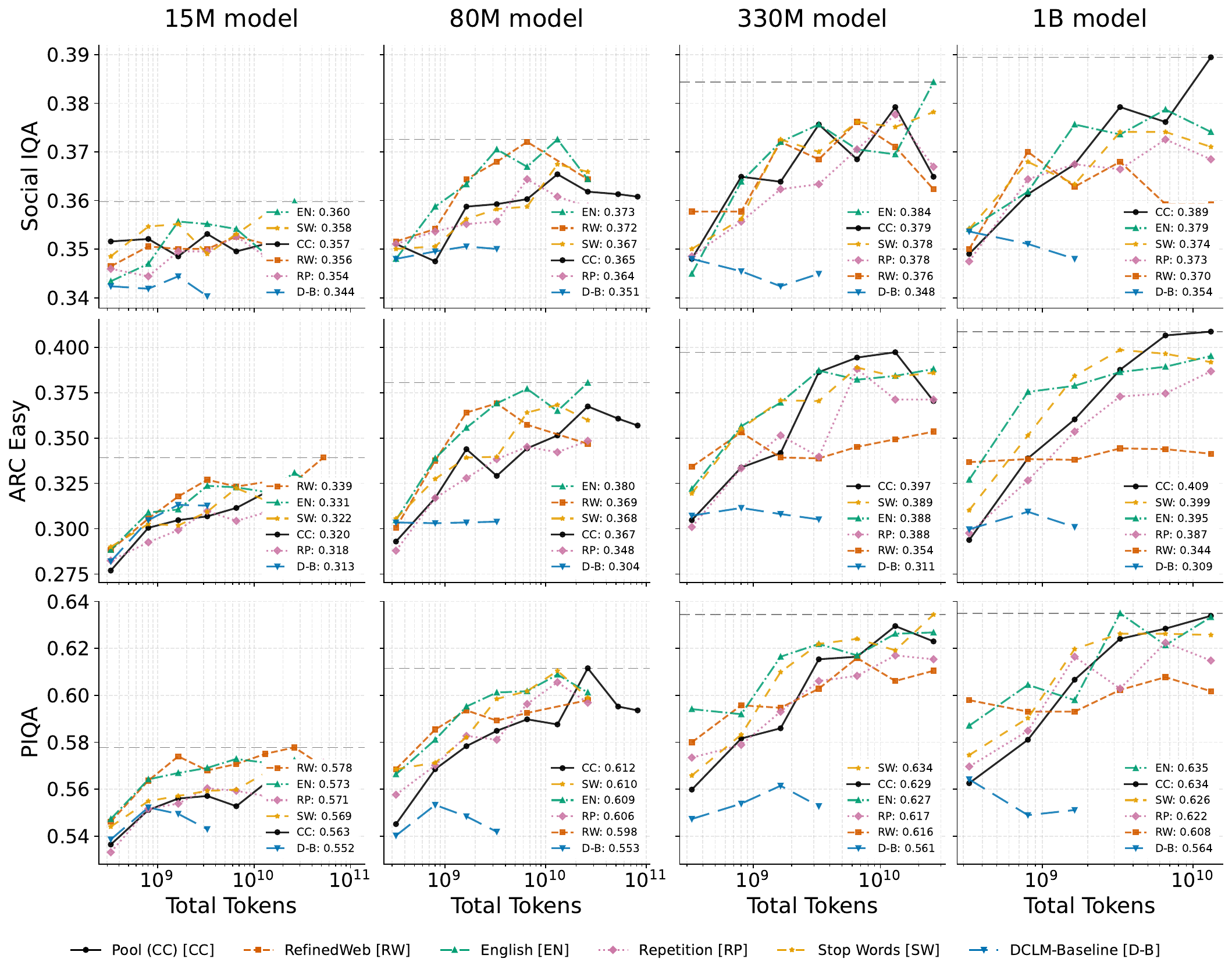}
    \caption{Ablation of $670$M token CC pool and five filtered versions. Each plot is a different model size and the total tokens x-axis corresponds to the number of gradient steps taken (with epoching).}
    \label{fig:filters-600m-pool-benchmarks}
\end{figure}

\begin{figure}[t]
    \centering
    \includegraphics[width=1\linewidth]{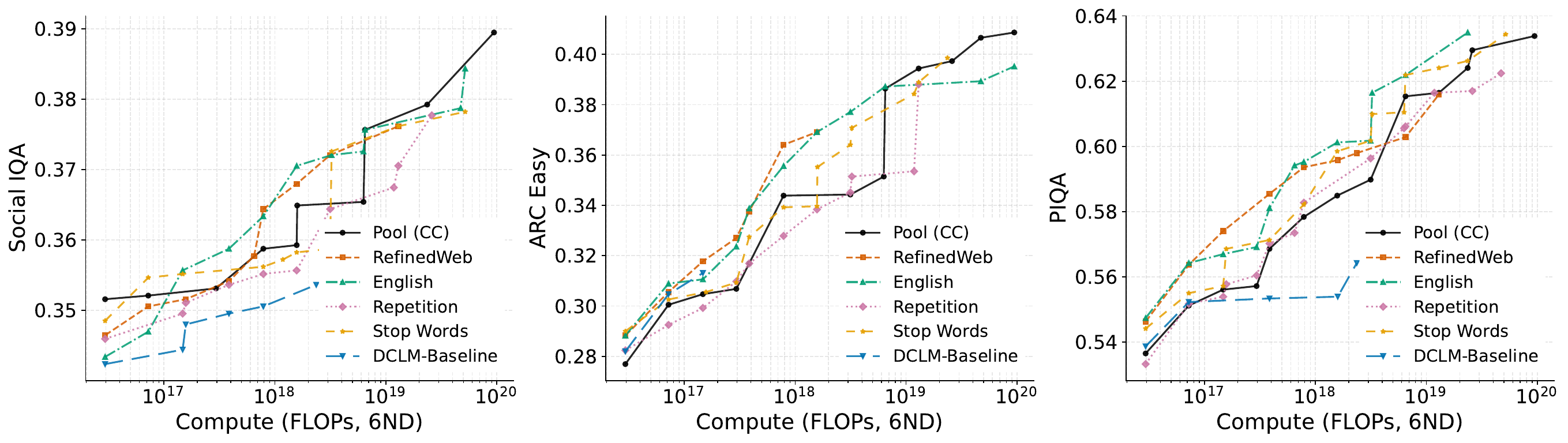}
    \caption{Pareto frontier of compute vs. benchmark performance for CC pool and filtered datasets. The frontier is formed with the same runs as in Figure~\ref{fig:filters-600m-pool-benchmarks}.}
    \label{fig:pareto-600m-pool-benchmarks}
\end{figure}

Finally, we address the potential confounding in Section~\ref{sec:filtering} when we used the DCLM-Baseline filter on the 670M Common Crawl pool, which retains roughly 2\% of the data and potentially results in severe data scarcity with respect to model size. While we did train a very small 15M parameter model in that setting, and note that no matter the subset size, DCLM-Baseline will always be about 2 orders of magnitude smaller than the pool, we provide an experiment here where we instead use 100M DCLM-Baseline tokens. This increases the size by roughly an order of magnitude. Figure~\ref{fig:dclm10x} adds this artificially-increased DCLM-Baseline to Figure~\ref{fig:filters-600m-pool}, and Figure~\ref{fig:pareto-dclm-10x} adds it to the Pareto curve of Figure~\ref{fig:pareto-670m-pool}. Even though the dataset now has more tokens than in the RefinedWeb subset, the pool and looser variants still outperform it with sufficient model size and training. 

\begin{figure}[t]
    \centering
    \includegraphics[width=1\linewidth]{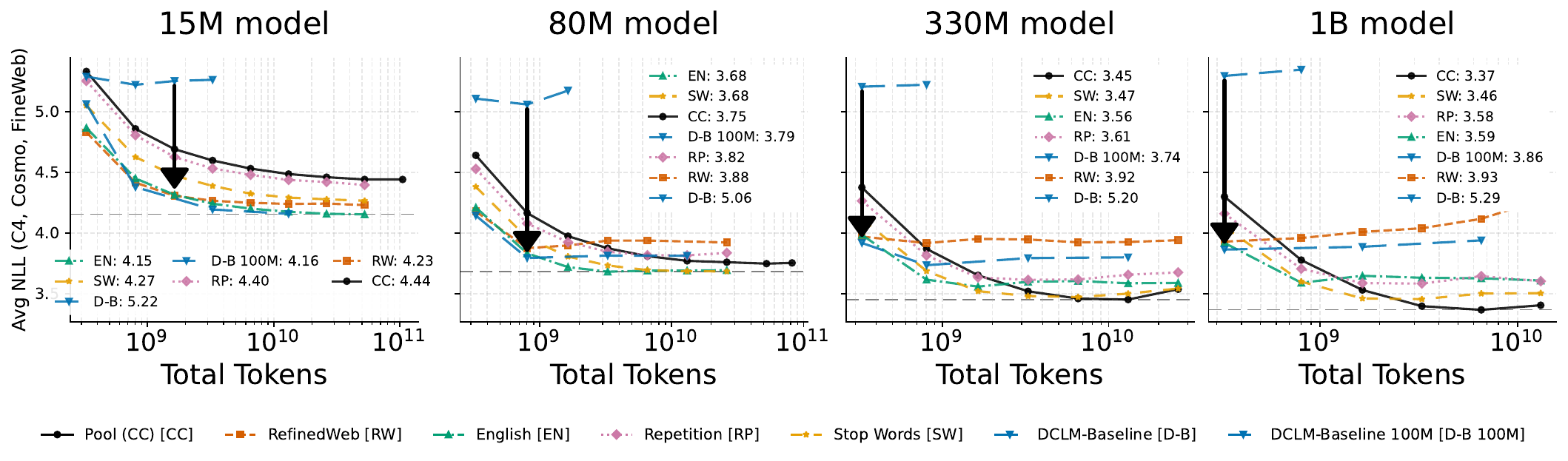}
    \caption{$670$M token CC pool and five filtered versions. Each plot is a different model size and the total tokens x-axis corresponds to the number of gradient steps taken (with epoching). The arrow shows the change in DCLM-Baseline performance with about an order of magnitude more tokens. }
    \label{fig:dclm10x}
\end{figure}

\begin{figure}[t]
    \centering
    \includegraphics[width=.5\linewidth]{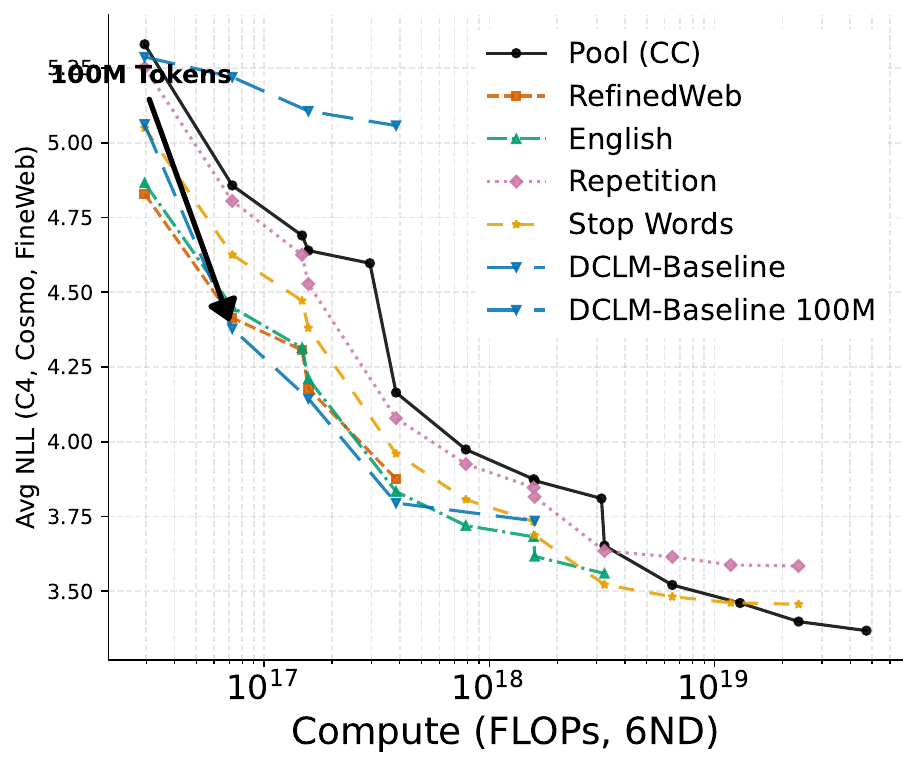}
    \caption{Pareto frontier of compute vs. average negative log-likelihood for CC pool and filtered datasets. The frontier is formed with the same runs as in Figure~\ref{fig:filters-600m-pool}. The arrow shows the change in DCLM-Baseline performance with about an order of magnitude more tokens.}
    \label{fig:pareto-dclm-10x}
\end{figure}

\begin{figure}[t]
    \centering
    \includegraphics[width=1\linewidth]{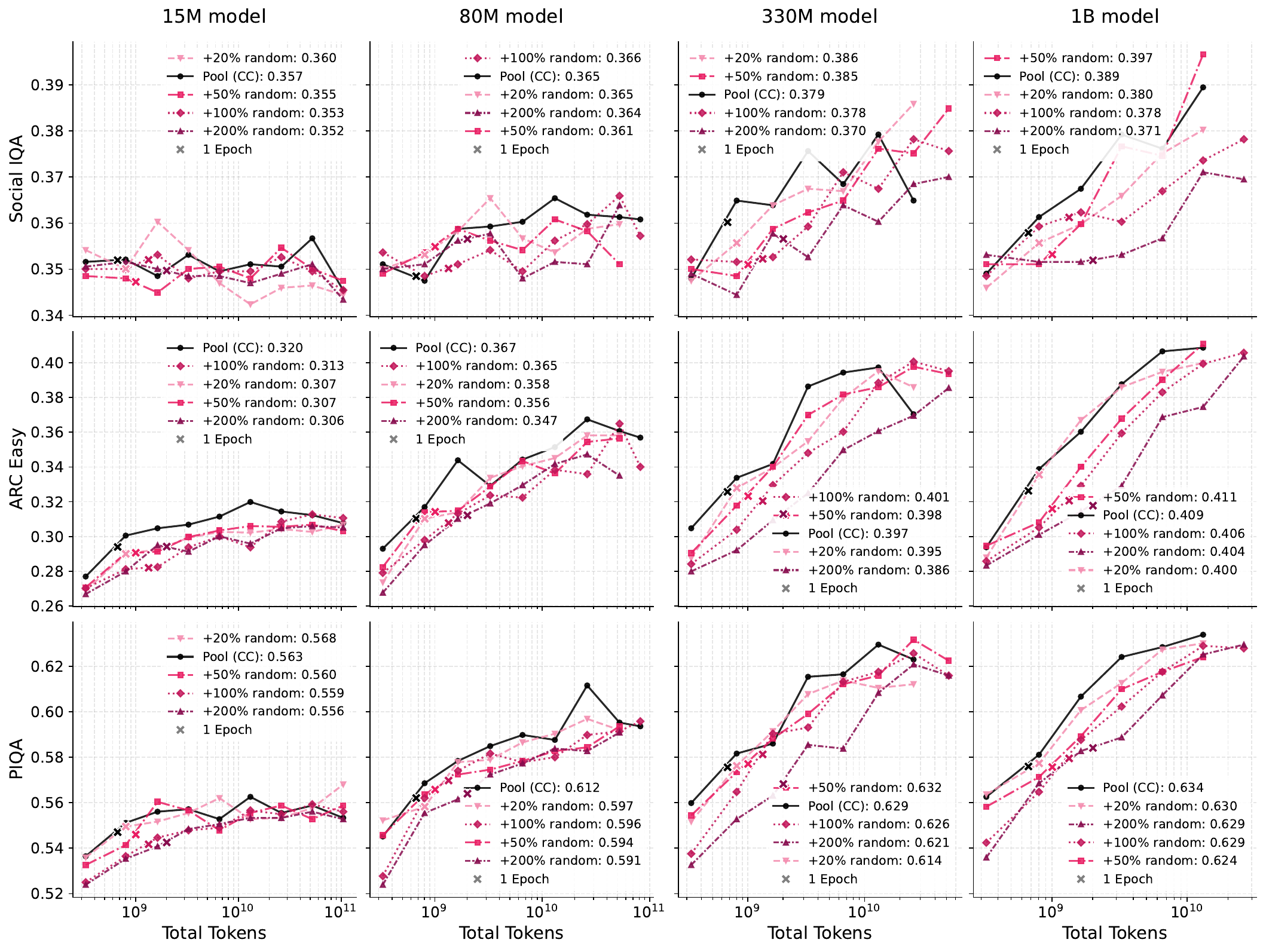}
    \caption{670M CC pool and random injection datasets. Each row is a downstream benchmark.}
    \label{fig:random-600m-pool-benchmarks}
\end{figure}

\begin{figure}[t]
    \centering
    \includegraphics[width=1\linewidth]{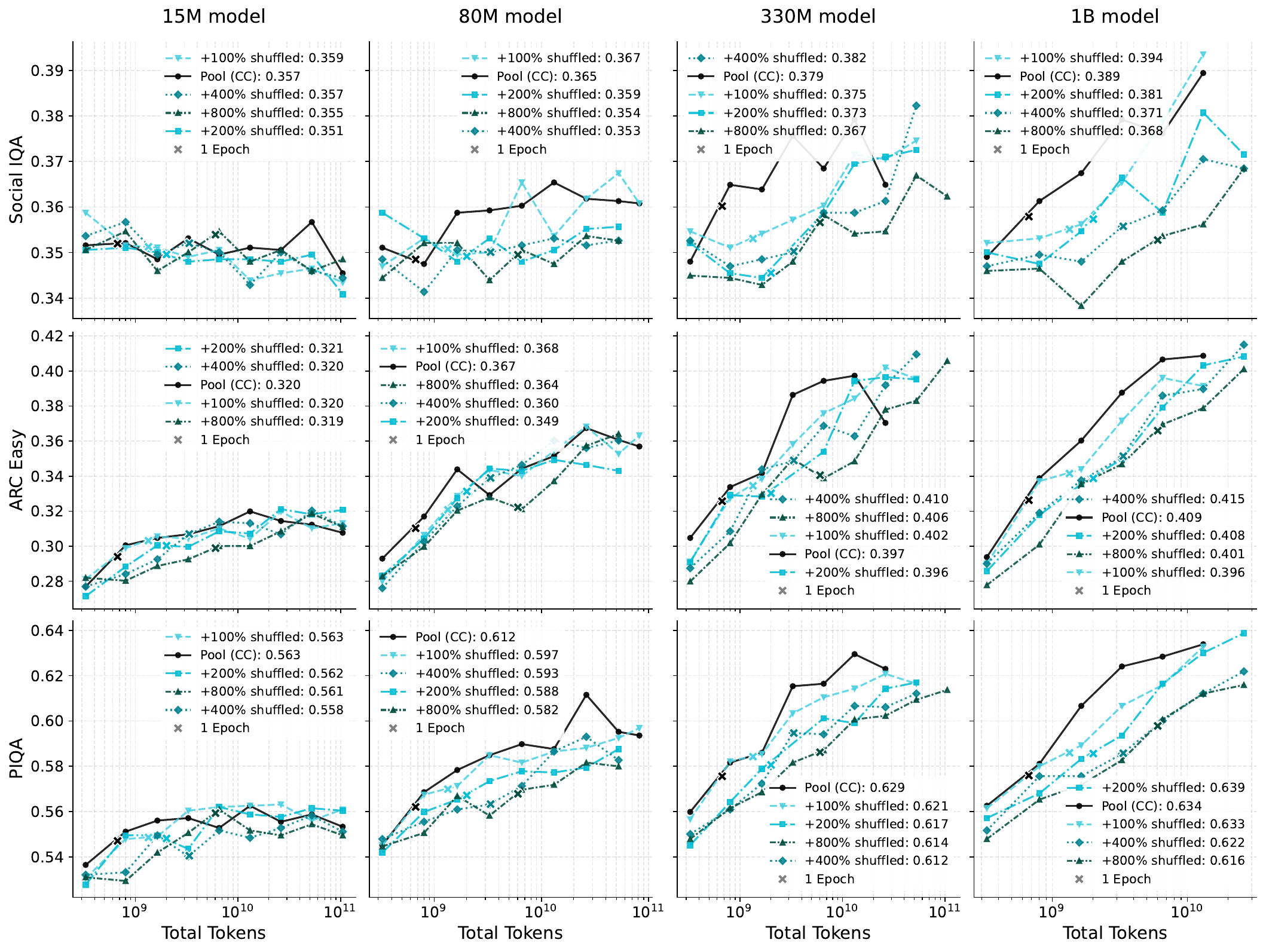}
    \caption{670M CC pool and shuffled-word injection datasets. Each row is a downstream benchmark.}
    \label{fig:shuffled-600m-pool-benchmarks}
\end{figure}

\clearpage

\section{Proofs and Additional Theory}\label{app:theory}

We now restate Proposition~\ref{thm:rank-necessity} from Section~\ref{sec:theory} and give its proof.

Consider predicting vector-valued outputs $y$
(tokens) using a rank $r$ linear transformation of an input $x$.
Assume the pairs $(x, y) \in \Rset^d \times \Rset^m$ come from one of $k$
tasks, where task $i$ occurs with probability $p_i > 0$ and generates $Y =
u_{\star,i}\, v_{\star,i}^\top X_i + \xi$ for independent mean-zero noise $\xi \in
\Rset^m$, where $\Sigma_i = \E[X_i X_i^\top]$ satisfy $\tr(\Sigma_i
\Sigma_{i'}) = 0$ for $i \neq i'$, so that tasks have orthogonal inputs:
one may exactly separate them.
We assume without loss of generality that $v_{\star, i} \in \range(\Sigma_i)$. 
The next proposition
follows.

\begin{proposition}[Rank Necessity under Orthogonal Inputs]
  Let the conditions above hold and $M_{\star,i} = u_{\star,i}\, v_{\star,i}^\top$,
  and define $M_\star = \sum_{i=1}^k M_{\star,i}$
  and $\Sigma = \sum_{i=1}^k p_i\Sigma_i$.
  If $\sigma_1 \ge \cdots \ge
  \sigma_\rho > 0$ are the $\rho \leq k$ positive singular values of
  $M_\star\Sigma^{1/2}$, then for any model rank $r$
  \[
  \min_{\substack{U \in \Rset^{m \times r}\\ V \in \Rset^{d \times r}}} \E\!\big[\|Y - UV^\top X\|^2\big] = \sum_{j=r+1}^{\rho} \sigma_j^2 + \E\!\big[\|\xi\|^2\big],
  \]
  where the sum evaluates to $0$ if $r \ge \rho$.
\end{proposition}

\begin{proof}
Let $M = UV^\top$. We first decouple the noise $\xi$:
\begin{align*}
  \E\big[\|Y - MX\|^2\big]
  &= \sum_{i=1}^k p_i\,\E\big[\|M_{\star,i}\,X_i + \xi - M\,X_i\|^2\big] \\
  &= \underbrace{\sum_{i=1}^k p_i\,\E\big[\|(M_{\star,i} - M)X_i\|^2\big]}_{=: g(M)}
    + \E\big[\|\xi\|^2\big],
\end{align*}
where the scalar cross-term $2\E[\xi^\top (M_{\star,i} - M)X_i] = 0$ vanishes by independence and zero mean. Since the noise term is independent of $M$, it suffices to minimize $g(M)$ over matrices of rank at most $r$. We rewrite the multi-task objective into a single-target objective:
\begin{align*}
  g(M) &= \sum_{i=1}^k p_i\,\tr\!\big((M_{\star,i} - M)\,\Sigma_i\,(M_{\star,i} - M)^\top\big) \\ 
  &= \sum_{i=1}^k p_i\,\tr\!\Big(\big(M_{\star} - M - \sum_{j\neq i}M_{\star, j}\big)\Sigma_i\big(M_{\star} - M - \sum_{j\neq i}M_{\star, j}\big)^\top\Big) \\
  &= \tr\!\big((M_{\star} - M)\,\Sigma\,(M_{\star} - M)^\top\big) \\
  &= \big\|(M_\star - M)\,\Sigma^{1/2}\big\|_F^2,
\end{align*}
where the cross terms in the second step vanish by $M_{\star, j}\Sigma_i = u_{\star, j} v_{\star, j}^\top \E[X_i X_i^\top] = 0$ for $i\neq j$ since $v_{\star, j}^\top X_i = 0$ almost surely.

Let $A = M_\star\Sigma^{1/2}$ have positive singular values $\sigma_1 \ge \cdots \ge \sigma_\rho$. The rank-constrained minimization reduces to
\begin{equation*}
  \min_{\rank(M) \le r} g(M) = \min_{\rank(M)\le r} \|A - M\Sigma^{1/2}\|_F^2.
\end{equation*}
For any $M$ with $\rank(M) \le r$, the matrix $B = M\Sigma^{1/2}$ has rank at most $r$. By the Eckart--Young--Mirsky theorem, the squared Frobenius distance between $A$ and any rank-$r$ matrix $B$ is at least $\sum_{j=r+1}^\rho \sigma_j^2$. This lower bound is exactly achievable: let $A_r$ be the rank-$r$ truncated SVD of $A$. Because $A$ is formed by right-multiplying by $\Sigma^{1/2}$, its row space and therefore the row space of $A_r$ lies entirely within the $\range(\Sigma^{1/2})$. Thus, setting $M = A_r(\Sigma^{1/2})^\dagger$ yields a valid matrix with $\rank(M) \le r$ that satisfies $M\Sigma^{1/2} = A_r$. The minimum achievable excess loss is therefore $\sum_{j=r+1}^\rho \sigma_j^2$, which vanishes if and only if $r \ge \rho$.
\end{proof}

\subsection{Theoretical conditions for filter improvement}

We now give a simple model that explains when filtering can improve or degrade performance.
To understand this, we hypothesize that a sufficiently trained large model's conditional distributions can be defined by a similarity measure $s \colon \mathcal{X} \times \mathcal{X} \to \Rset_+$ over test inputs $x \in \mathcal{X}$ and train inputs $x_i \in \mathcal{X}$ from a training dataset $D=\{(x_i, y_i)\}_i$:
\[\P_D(y \mid x) := \sum_{i\in D} \frac{s(x, x_i) }{\sum_{j\in D} s(x, x_j )}\1_{y_i = y}. \]
The conditional distribution is the fraction of examples in $D$ with the same label $y$, weighted by $s$. According to the definition, we can affect the model's prediction at a given test input $x$ by including a similar $x'$ in the training dataset $D$.\ignore{ If the corresponding label $y'$ does not match a target $y^\star$, the cross-entropy loss $-\log \P_D(y^\star\mid x)$ on the test example $x$ will increase. If we include unrelated documents with $s(x, x') = 0$, the loss is unaffected. }

Let us consider the error (in KL divergence) of this predictor $\P_D$ compared
to a predictor using filtered data $\P_{\phi\circ D}$, where
$\phi:\mathcal X\times\mathcal Y\to\{0,1\}$ is a filter and
$\phi\circ D\subseteq D$.  We use the notation $D_{|y}$ to denote the restriction of $D$ to examples $(x_i, y_i)$ with $y_i=y$ and $D_{|\neq y}$ to denote the restriction of $D$ to examples $(x_i, y_i)$ with $y_i\neq y$. Expectations are defined with respect to an $s(x, \cdot)$-weighted dataset; we assume the relevant weighted masses are nonzero so that the displayed
conditional distributions and expectations are well-defined. We find a simple characterization of the error difference.
\begin{fact}[Characterization of Filter Improvement]
\label{fact:improvement}
    Given Dirac target conditional $\P_\text{t}(\cdot \mid x)$ with all mass on $y^\star$\ignore{, dataset $D = \{(x_i, y_i)\}_i$, and filter function $\phi \colon \sX \times \sY \to \{0, 1\}$, }, the improvement of $\P_{\phi\circ D}$ with respect to $\P_D$ in KL divergence to $\P_t$ is 
    \[
    KL( \P_t \mid\mid \P_{D}) - KL(\P_{t} \mid\mid  \P_{\phi\circ D}) =  -\log \paren*{\P_D(y^\star \mid x) + (1-\P_D(y^\star \mid x)) \frac{\E_{D_{|\neq y^\star}} \bracket{\phi(X, Y)} }{\E_{D_{|y^\star}} \bracket{\phi(X, Y)} }}.
    \]
\end{fact}
\begin{proof} In the following, we drop the first argument to $s(\cdot, \cdot)$ to simplify notation. We first simplify the difference using the definition of KL divergence: 
    \begin{align*}
        & KL( \P_{t} \mid\mid \P_{D}) - KL(\P_{t} \mid\mid  \P_{\phi\circ D})  
        =  \sum_{y\in\sY} \P_{t}(y\mid x) \log \frac{\P_{\phi \circ D}(y\mid x)}{\P_D(y\mid x)}. 
    \end{align*}
    We analyze the likelihood ratio:
\begin{align*}
    \frac{\P_D(y\mid x)}{\P_{\phi \circ D}(y\mid x)} & = \frac{\sum_{i\in D, y_i = y} \frac{s(x_i)}{\sum_{j\in D} s( x_j)}}{\sum_{i\in D, y_i = y} \frac{s(x_i)\phi(x_i, y_i)}{\sum_{j\in D} s( x_j)\phi(x_j, y_j)}} \\
    & = \frac{\sum_{j\in D} s( x_j)\phi(x_j, y_j)}{\sum_{j\in D} s( x_j)}\frac{\sum_{i\in D, y_i = y} s(x_i)}{\sum_{i\in D, y_i = y}s(x_i)\phi(x_i, y_i)} \\
    & = \frac{\E_{D} \bracket{\phi(X, Y)} }{\E_{D_{|y}} \bracket{\phi(X, Y)} } \\
    & = \frac{\P_{D}(y \mid x)\E_{D_{|y}} \bracket{\phi(X, Y)} + (1-\P_{D}(y \mid x)) \E_{D_{|\neq y}} \bracket{\phi(X, Y)} }{\E_{D_{|y}} \bracket{\phi(X, Y)} } \\
    & = \P_{D}(y \mid x) + \left(1-\P_{D}(y \mid x)\right) \frac{\E_{D_{|\neq y}} \bracket{\phi(X, Y)} }{\E_{D_{|y}} \bracket{\phi(X, Y)} }. 
\end{align*}
Plugging this back in, we find that the general difference is
    \begin{align*}
     - \sum_{y\in\sY} \P_{t}(y\mid x) \log \paren*{\P_D(y \mid x) + (1-\P_D(y \mid x)) \frac{\E_{D_{|\neq y}} \bracket{\phi(X, Y)} }{\E_{D_{|y}} \bracket{\phi(X, Y)} }}.
    \end{align*}
    Fact~\ref{fact:improvement} follows as a special case by setting $\P_t(y \mid x) = \1_{y = y^\star}$.
\end{proof}
The fact shows that two terms appear in the difference: the prevalence of the
label $y^\star$ in the original dataset $\P_D(y^\star\mid x)$ and a measurement
of filter performance via the ratio of the false positive rate to the true
positive rate,
\[
\frac{\E_{D_{|\neq y^\star}}[\phi(X,Y)]}{
\E_{D_{|y^\star}}[\phi(X,Y)]}.
\]
When $\P_D(y^\star\mid x)<1$, filtering improves the KL if and only if this
ratio is less than 1. If the prevalence is already high, there is little
improvement possible, and otherwise $\phi$ must distinguish correct from
incorrect labels on the $s$-weighted dataset. In the case of CC,
Table~\ref{tab:mmlu_averages} suggests that the prevalence on select MMLU
subjects is already high. In cases of strong filtering, e.g. removing $99\%$ of
the data including all $x'$ with high $s(x,x')$, the true positive rate may approach zero, making the ratio large and the KL worse.

\newpage

\end{document}